\documentclass[10pt,letterpaper]{article}

\usepackage{cogsci}

\usepackage{blindtext}
\usepackage{amsfonts}
\usepackage{mathrsfs}
\usepackage{amsmath}
\usepackage{amssymb}
\usepackage{amsthm}
\usepackage{algorithm}
\usepackage{algpseudocode}
\usepackage{graphicx}

\usepackage{dblfloatfix}

\cogscifinalcopy % Uncomment this line for the final submission 

\usepackage{pslatex}
\usepackage{apacite}
\usepackage{float} % Roger Levy added this and changed figure/table
\title{Computational Complexity of Segmentation}
 
\author{{\large \bf Federico Adolfi (federico.adolfi@esi-frankfurt.de)} \\
  Ernst-Strüngmann Institute for Neuroscience in Cooperation with Max-Planck Society, Germany \\
  School of Psychological Science, University of Bristol, UK
    \AND {\large \bf Todd Wareham (harold@mun.ca)} \\
  Department of Computer Science, Memorial University of Newfoundland, Canada
  \AND {\large \bf Iris van Rooij (iris.vanrooij@donders.ru.nl)} \\
  Donders Institute for Brain, Cognition, and Behaviour, Radboud University, The Netherlands \\ Department of Linguistics, Cognitive Science, and Semiotics \& Interacting Minds Centre, Aarhus University, Denmark}

\makeatletter
\newtheoremstyle{indented}
  {3pt}% space before
  {3pt}% space after
  {\addtolength{\@totalleftmargin}{2.5em}
   \addtolength{\linewidth}{-1.5em}
   \parshape 1 1.5em \linewidth}% body font
  {}% indent
  {\bfseries}% header font
  {.}% punctuation
  {.5em}% after theorem header
  {}% header specification (empty for default)
\makeatother

\theoremstyle{indented}
\newtheorem{theorem}{Theorem}

\theoremstyle{indented}

\theoremstyle{indented}
\newtheorem{lemma}{Lemma}

\theoremstyle{indented}
\newtheorem{definition}{Definition}

\theoremstyle{remark}
\newtheorem*{remark}{Remark}

\begin{document}

\maketitle

\begin{abstract}

Computational feasibility is a widespread concern that guides the framing and modeling of biological and artificial intelligence. The specification of cognitive system capacities is often shaped by unexamined intuitive assumptions about the search space and complexity of a subcomputation. However, a mistaken intuition might make such initial conceptualizations misleading for what empirical questions appear relevant later on. We undertake here computational-level modeling and complexity analyses of \textit{segmentation} — a widely hypothesized subcomputation that plays a requisite role in explanations of capacities across domains — as a case study to show how crucial it is to formally assess these assumptions. We mathematically prove two sets of results regarding hardness and search space size that may run counter to intuition, and position their implications with respect to existing views on the subcapacity.

\end{abstract}

% less space between paragraphs and equations
\setlength{\belowdisplayskip}{8pt} \setlength{\belowdisplayshortskip}{4pt}
\setlength{\abovedisplayskip}{8pt} \setlength{\abovedisplayshortskip}{4pt}

\section{Introduction}

% Apacite prefix syntax example: \cite<see, e.g., >{vanrooijCognitionIntractabilityGuide2019}

Cognitive scientists routinely invoke subcapacities in decompositional efforts to reverse-engineer fully-fledged capacities of minds, brains and machines \cite{cumminsHowDoesIt2000,milkowskiReverseengineeringCognitiveScience2013}. For instance, speech processing is presumed to decompose into, among other things, segmentation and decoding, and action understanding into parsing, predicting and goal inference. These \textit{subcomputations} are thought to tackle certain \textit{problems} that the cognitive system faces to behave appropriately in the world.
%\footnote{We bear in mind various distinctions: 1) \textit{problems} as they are intuited by researchers and \textit{computational problems} as formalized in computer science; 2) hypothesized \textit{capacities} of cognitive systems, the real-world capacities they allude to, and researchers' \textit{explanations} of them which may include conjectured \textit{(sub)computations}. We use computational-level theorizing and analysis as a way to explain cognitive capacities \cite<cf.>{marr1982vision}, initially related to informal problems, through the formal definition of computational problems and associated computations. We only have access to the real-world capacities through this explanatory process.}

Problems that originally show up in one domain (e.g., speech processing) are subsequently encountered in other domains (e.g., action understanding), and so the conceptual apparatus naturally carries over. As a result, cognitive scientists may come to view, e.g., the problem of segmenting speech as analogous to the problem of parsing actions. Researchers can then transfer ideas across the domains, adopting and adapting similar subcomputations in their explanations of the different capacities. What is passed along, however, will include latent (and possibly mistaken) notions about the computational properties of these problems as well. For instance, if a cognitive scientist believes that the search space of speech segmentation is large (combinatorially complex) and that this makes the problem hard, then by analogy, the same can be inferred about the parsing problem in action understanding.  

% \cite<e.g., >[for segmentation in spatial navigation and episodic memory]{brunecBoundariesShapeCognitive2018}

Once such initial framing of a cognitive (sub)capacity is adopted, it completely shapes the kinds of empirical questions that appear relevant and in so doing determines the course of research programs across disciplines and cognitive domains. Crucially, the assumptions that gave rise to the initial framing are seldom examined formally, and since they are taken for granted as background commitments, empirical tests are not designed to bear on them. These foundational oversights can sidetrack researchers into directions that will be largely immune to empirical corrective feedback later on.

To illustrate how crucial it is to formally assess the validity of intuitive assumptions about problem properties, and what can go astray if one doesn't, we undertake here a formal examination of an example subcapacity.  Our case study is \textit{Segmentation}\footnote{Segmentation relates closely to computations whose names vary depending on time period, cognitive domain, and theoretical framework: chunking, sampling, discretization, integration, grouping, packaging, quantization, sequencing, segregation, parsing, temporal pooling, temporal gestalt, boundary placement, temporal attention.}.  This subcapacity figures ubiquitously in explanations of real-world cognitive capacities such as speech recognition, music perception, active sensing, event memory, temporal attention, action processing, and sequence learning. We focus on two classes of assumptions about its computational properties: i) the search space is excessively complex and ii) this makes the segmentation problem intrinsically hard. 

To formally assess the theoretical viability of these assumed properties, we develop a formalization of the (intuitive) segmentation problem at Marr's \citeyear{marrVisionComputationalInvestigation1982} computational level. Next, we submit this formalization to a mathematical analysis to assess the size of its search space, its computational hardness, and its possible sources of complexity using tools from computational complexity theory \cite{gareyComputersIntractabilityGuide1979,aroraComputationalComplexityModern2009,vanrooijCognitionIntractabilityGuide2019}.
As our results may run counter to intuition, we end with a word of caution regarding the general non-intuitiveness of the computational properties of hypothesized cognitive problems. 
 
 %The remainder of the paper is structured as follows. First, we introduce the problem of segmentation as it is conceptualized in the literature across cognitive domains. Next, we develop a formalization that distills the essence of such specifications. We then survey and synthesize the core intuitive assumptions about the computational properties of the problem. Finally, we present proofs that speak to the validity of these assumptions, and discuss the implications for research on segmentation and subcomputations more broadly.

\section{Conceptualization of segmentation}

In order to rigorously examine computational assumptions, we need a mathematical formalization of the problem that can be submitted to further analyses. This computational-level model, in turn, should capture key aspects of the theorized cognitive capacity. To that end, in this section we synthesize conceptualizations of the segmentation problem as it appears in various cognitive domains.

\subsection{Informal definitions: segmentation as a fundamental subcomputation}

``How the brain processes sequences is a central question in cognitive science and neuroscience" \cite{jinLowfrequencyNeuralActivity2020}. A substantial amount of information available to the cognitive system is ``continuous, dynamic and unsegmented" \cite{zacksHumanBrainActivity2001}. The purpose of the segmentation process is, then, ``to generate elementary units of the appropriate temporal granularity for subsequent processing" \cite{giraudCorticalOscillationsSpeech2012}. Succinctly, ``[t]he central nervous system appears to `chunk' time" \cite{poeppelAnalysisSpeechDifferent2003}.

Several subfields of the cognitive and brain sciences have proposed segmentation as a key subcomputation. \textit{Active listening} (cf. active sensing) casts it as ``the selection of internal actions, corresponding to the placement of [...] boundaries" \cite{fristonActiveListening2021}, ``to sample the environment" \cite{poeppelSpeechRhythmsTheir2020}. \textit{Event cognition} similarly defines it as ``the process of identifying event boundaries [...] a concomitant component of normal event perception" \cite{zacksHumanBrainActivity2001}. In \textit{episodic memory}, it is ``the process by which people parse the continuous stream of experience into events and sub-events [for] the formation of experience units" \cite{jeunehommeEventSegmentationTemporal2018}. Central to \textit{music perception}, it features as determining the ``perceptual boundaries of temporal gestalts" \cite{tenneyTemporalGestaltPerception23} and ``entails the parsing into chunks" \cite{farboodDecodingTimeIdentification2015,tillmannMusicLanguagePerception2012}. The \textit{speech recognition} literature describes it as the core process of ``segmenting the continuous speech stream into units for further perceptual and linguistic analyses" \cite{tengSpeechFineStructure2019}, where it ``allows the listener to transform [the] signal into segmented, discrete units, which form the input for subsequent decoding steps" \cite{poeppelSpeechRhythmsTheir2020}. In \textit{action processing}, ``[a] fundamental problem observers must solve [...] is segmentation [...] Identifying distinct acts within the dynamic flow of motion is a basic requirement for engaging in further appropriate processing" \cite{baldwinSegmentingDynamicHuman2008}.

Such ubiquitousness has been suggestive that the capacity ``appeals to general principles the brain may use to solve a variety of problems" \cite{fristonActiveListening2021,himbergerPrinciplesTemporalProcessing2018}. ``[M]any sequence-chunking tasks share common computational principles. [E.g.,] to find and encode the chunk boundaries" \cite{jinLowfrequencyNeuralActivity2020}. Segmentation as a subcomputation appears across processing hierarchies as well, even when the world is relatively static: ``[it] exists at multiple layers within a given problem" \cite{wybleTemporalSegmentationFaster2019}. The downstream operations on segments that partially determine optimal segmentation play similar roles but vary with cognitive domain and modeling framework.

\textit{Segmentation}, concisely, is then a fundamental subcomputation whose requisite role across cognitive domains and processing hierarchies is to determine, given a sequence representation, the optimal boundary placement with respect to a downstream computation over segments.
		
\section{Formalization of segmentation}

A succinct, yet informal, definition of segmentation can be stated by verbally specifying the inputs and outputs of the conjectured subcomputation. 
\begin{quote}
\textsc{segmentation (informal)}\\
    \textit{Input}: A sequence and a downstream process that, for any given segment of the sequence, can evaluate its quality relative to domain-specific criteria. \\
    \textit{Output}: The best\footnote{Without loss of generality, here `best' could be replaced by `good enough', and our formal results would still apply.} segmentation of the sequence with respect to criteria relevant for the downstream process.
\end{quote}

\noindent With this sketch in mind (see Fig. \ref{segmentation-schematic} for a schematic), we develop the formal definition of the computational-level model.

\begin{figure}[h]
\vspace{-0.2cm}
\begin{center}
    \includegraphics[scale=0.45]{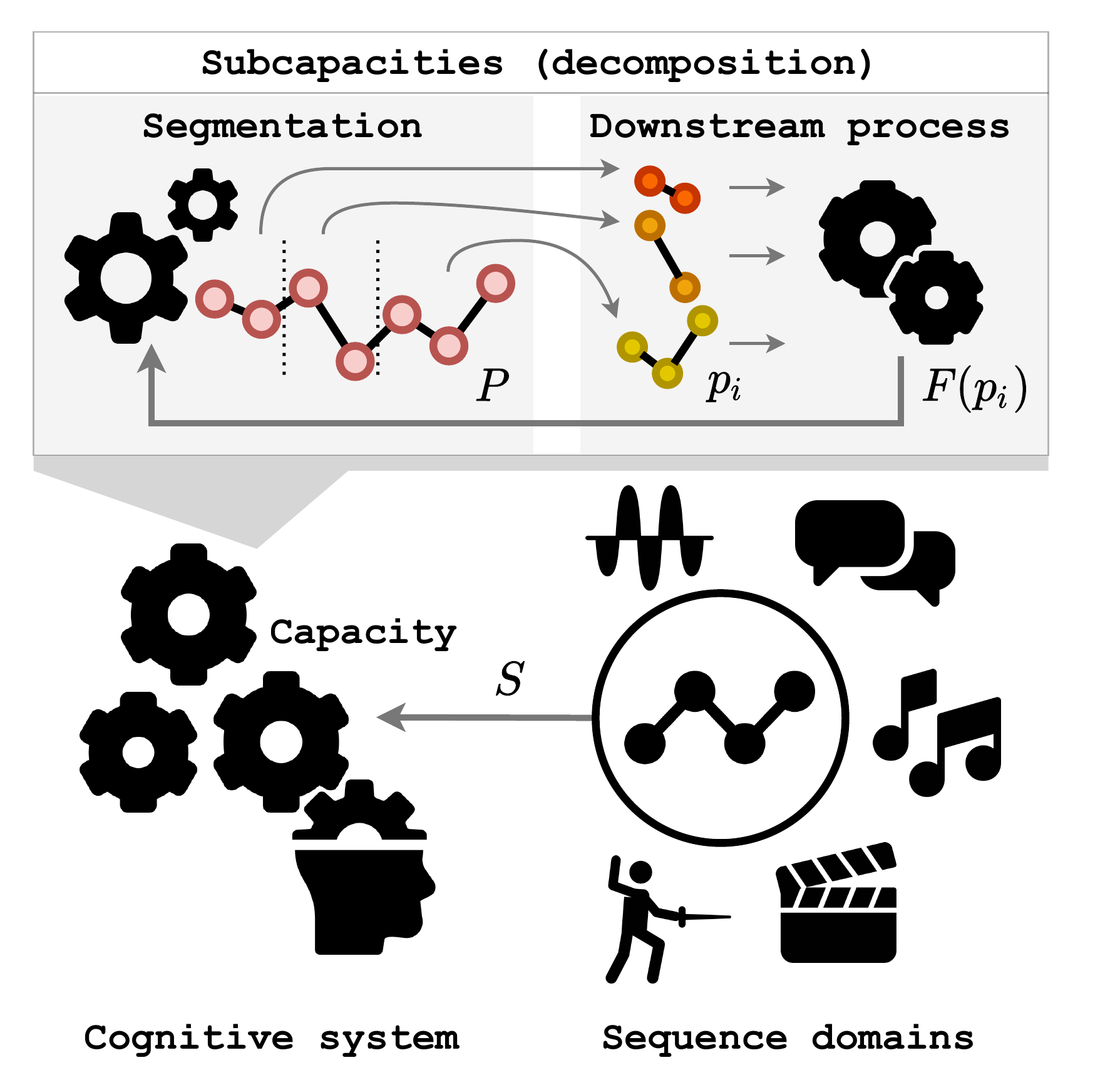}
\end{center}
\vspace{-0.6cm}
\caption{Segmentation is a core subcomputation in domains including sound, speech, music, action and event processing (bottom). Segmentation itself can be functionally decomposed such that domain-specific downstream processes inform which possible segmentations of a given sequence $S$ are best (top). Refer to the main text for definitions of $S$, $P$, $F$ and $p_i$.}
\vspace{-0.0cm}
\label{segmentation-schematic}
\end{figure}

We envision an input sequence $S = (s_1, s_2, ..., s_n)$ that captures the idea of a time-ordered representation the cognitive system must work with. Its origin could be sensory encoding at the periphery or deeper, more elaborate processes alike (e.g., an encoding of the acoustic envelope of speech or music, or a compressed representation of a visual scene). As instances of the segmentation problem appear throughout processing hierarchies, their  inputs vary in origin and nature. We model the sequence with according generality. Next, we pin down the notion of a downstream cognitive process that computes over segments $p \in \mathscr{P}$ (e.g., a decoder that maps speech segments to phonemes, or a module that maps scene segments to action meanings). Our formalization is agnostic as to what these domain-specific processes, and  theoretical frameworks used to model them, might be. We aim for generality and simply model, with a function $F: \mathscr{P} \mapsto \mathbb{Z^+}$ (over a possibly infinite domain) available at the input, the idea that the process is capable of guiding the placement of boundaries. This is achieved by reporting back some (discretized) aspect of its performance $F(p) \in \mathbb{Z^+}$ (e.g., label probability, likelihood w.r.t. generative model, depending on framework). The desired output — a useful segmentation scheme — is modeled as a collection $P$ of disjoint segments jointly making up the input sequence, whose overall appropriateness $V(P)$ with respect to the downstream process is optimal. These modeling choices yield the following formalization.\footnote{For succinctness, we omit the set notation for sequences and subsequences, $\{(i, s_i), ..., (n, s_n)\}$, and we slightly abuse notation when using set operations directly on tuples.}

\begin{quote}
    \textsc{segmentation (formal)} \\
    \textit{Input}: a finite sequence $S = (s_1, s_2, ..., s_N)$ of length $N \in \mathbb{N}$, with $s_i \in \mathbb{Z}$ and a scoring function $F: \mathscr{P} \mapsto \mathbb{Z^+}$ that maps contiguous subsequences $p = (s_i, s_{i+1}, ..., s_{i+q})$ to a positive value $F(p)$. \\
    \textit{Output}: a segmentation of $S$ into contiguous subsequences, $P = ( (s_1, s_2, ...), ..., (..., s_{N-1}, s_N) )$, where segments are disjoint, $\forall p_i, p_j \in P: p_i \cap p_j = \emptyset$, and their concatenation yields the original sequence, $\bigcup_{i=1}^{|P|} p_i = S$, such that its overall value $V(P) = \sum_{p \in P} F(p)$ is maximum.\footnote{   We model segmentation as an optimization problem to keep with conceptual constraints but without loss of generality. This modeling choice makes our results an upper-bound on the complexity of the problem.}

\end{quote}

\section{Assumptions about segmentation}

To determine the course of our analyses, we survey  views on the computational properties of the segmentation problem. We illustrate with examples and synthesize core intuitions.

\subsection{Problem properties: segmentation as a computational challenge}

\subsubsection{Hardness and complexity.} Segmentation problems have been widely assumed to be computationally challenging. This is evidenced in explicit statements and in the `solutions' researchers propose after taking onboard certain beliefs about hardness. To illustrate: ``Speech recognition is not a simple problem. The auditory system must parse a continuous signal into discrete words" \cite{fristonActiveListening2021}. ``It is hard for a brain, and very hard for a computer" \cite{poeppelAnalysisSpeechDifferent2003}. ``[S]egmentation requires inference over the intractably large discrete combinatorial space of partitions." \cite{franklinStructuredEventMemory2020a}.

\subsubsection{Sources of complexity.} As is evident in researchers' descriptions, the hardness is attributed to the (presumed) combinatorial explosion involved in the number of possible segmentation schemes — the size of the problem search space is informally taken as the source of computational complexity. Again, to illustrate: ``Where should these candidate boundaries be placed? In an extreme case, we could place boundaries at every combination of time points [...] but that would be computationally inefficient given that we can reduce the scope of possibilities" \cite{fristonActiveListening2021}. ``The problem would be enormously complicated by the presence of so many candidates [...]" \cite{brentSpeechSegmentationWord1999}.
		
\subsubsection{Solutions for complexity.} Arguably as a consequence of coupling these intuitions with additional assumptions, the effectiveness of certain solutions has been taken for granted. ``From the computational perspective, the aim of research in segmentation [...] is to identify mechanisms [that] reduce these computational burdens by reducing the number of candidate[s]" \cite{brentSpeechSegmentationWord1999}. This position has motivated the search for bottom-up segmentation cues or top-down biases (e.g., priors) that would achieve, among other things, such a narrowing down \cite<e.g.,>{tengSpeechFineStructure2019,fristonActiveListening2021}. ``We suggest a different role [of cues] in which they are part of the [segmentation] (rather than decoding) process" \cite{ghitzaRoleThetaDrivenSyllabic2012}. For instance, researchers may observe environmental \cite{dingTemporalModulationsSpeech2017} and neural \cite{tengConcurrentTemporalChannels2017, tengThetaBandOscillations2017} regularities suggestive of segment-size constrained segmentation processes \cite{poeppelSpeechRhythmsTheir2020,poeppelAnalysisSpeechDifferent2003}.

\subsection{Core assumptions}
This survey reveals a core set of intuition-based assumptions about the computational properties of segmentation:
\vspace{-0.6em}
\begin{itemize}
    \itemsep-0.2em 
    \item Real-world sequences (e.g., speech, music, scenes, actions) and internal representations alike (e.g., memories of experiences) are ``complex, continuous, dynamic flows".
    \item The cognitive system needs to make use of discrete representations of segments that are appropriate (size- and content-wise) for downstream tasks.
    \item The problem is ``hard" — the obstacle being that there are ``too many" possible segmentations of a given sequence.
    \item Cognitive systems must reduce the possibilities somehow, e.g., via bottom-up cues and/or top-down biases.
\end{itemize}

\section{Computational complexity of segmentation}
It is generally non-obvious what problems are genuinely (as opposed to merely apparently) hard, which refinements will render a model tractable, or which restrictions will effectively reduce a search space. Intuitions about computational properties of problems are frequently mistaken, hence need to be validated against formal analyses \cite{vanrooijIdentifyingSourcesIntractability2008a}. 
%We do so here through the lens of theoretical computer science. 
This section presents a complexity analysis in two parts according to the assumed properties they examine: search space size, and problem hardness.

\subsection{Search space of segmentation}
We analyze the search space size as a possible source of hardness by envisioning a simple brute-force algorithm. If the number of candidate solutions grows polynomially
%\footnote{That is, a function upperbounded by $N^c$ where $N$ is the length of the sequence and $c$ is some constant.} 
 (i.e., upperbounded by $N^c$, where $N$ is the sequence length and $c$ is some constant), then such an algorithm would be tractable. 
 %It follows that the outcome of this thought experiment hangs entirely on what is revealed by the combinatorial structure of the search space. 
We describe this growth through combinatorial analysis; first for the unconstrained problem and then including various theoretically motivated constraints.

\subsubsection{Unbounded parts.}
When the size $q$ of the segments is not constrained other than by the length $N$ of the sequence, i.e., $q \in [1, N]$, all boundary placements are possible. Notice there is a bijection between binary strings of length $N-1$ and boundary placements in sequences of length $N$ (Fig. \ref{boundaries}).
    
\begin{figure}[h]
\vspace{-0.2cm}
\begin{center}
    \includegraphics[scale=0.3]{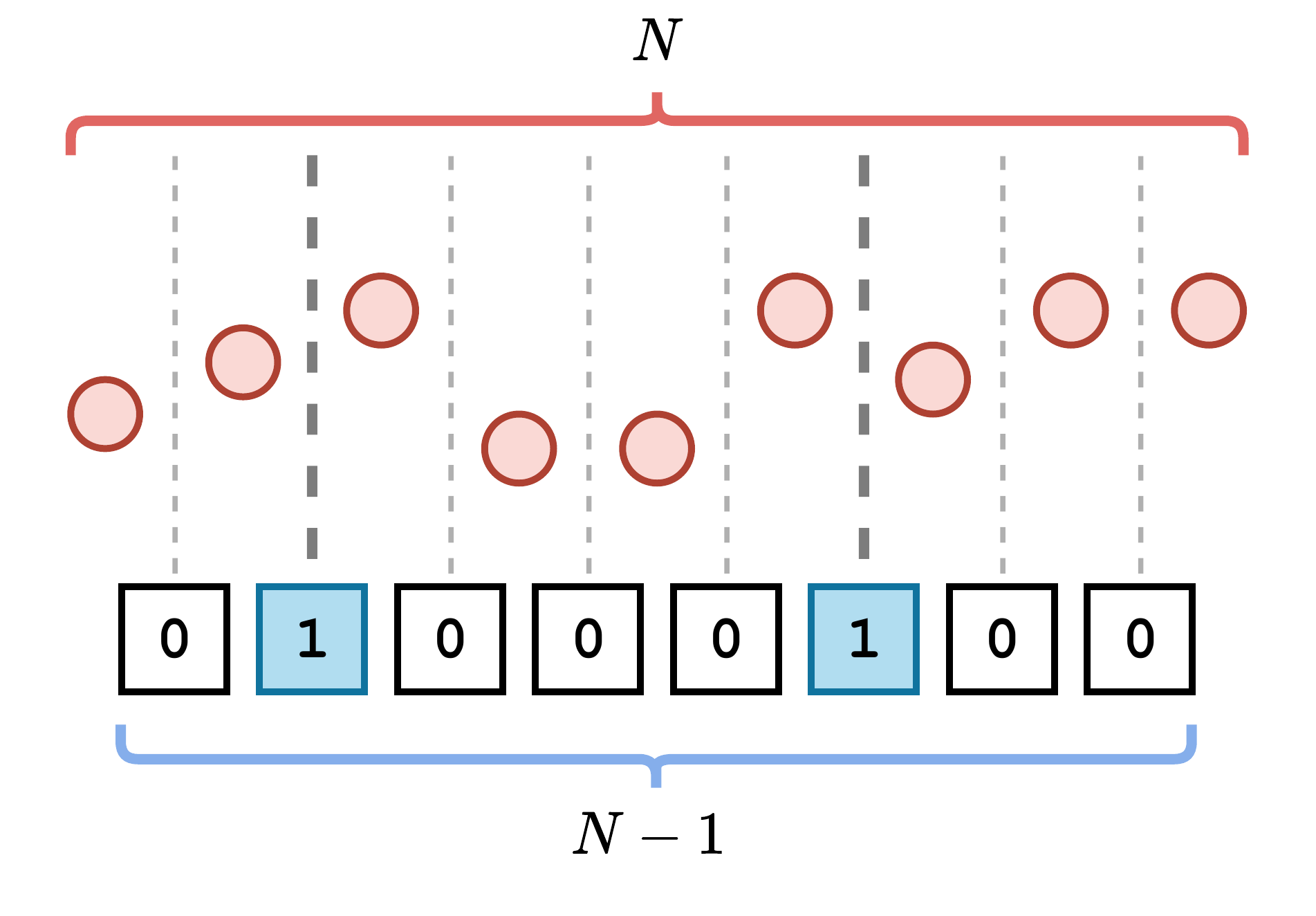}
\end{center}
\vspace{-0.6cm}
\caption{Binary strings encode boundary placements. A sequence of length $N$ (top) admits $N-1$ choices of the presence/absence of boundaries (bottom).}
\vspace{-0.1cm}
\label{boundaries}
\end{figure}

\noindent Since the number of possible binary strings of length $k$ is given by $2^k$, the number of possible segmentations that use unbounded parts grows as $2^{N-1}$ (i.e., exponentially).

\subsubsection{Segmentation as integer composition.} In order to incorporate various constraints in combinatorial analyses, we draw an analogy between segmentation and \textit{integer compositions} (Fig. \ref{integer-composition}). This enables us to take an analytic combinatorics approach \cite{flajoletSedgwick2009} to the latter and leverage the results to infer properties of the former.

\begin{definition}[Integer composition]
A composition of an integer $N$ is an ordered list $C = (p_{1}, p_{2}, ..., p_k)$ of positive integer parts $p_i \in \mathbb{N}^+$, such that $N = \sum_{p \in C} p$.
\end{definition}

\begin{figure}[H]
\vspace{-0.3cm}
\begin{center}
    \includegraphics[scale=0.50]{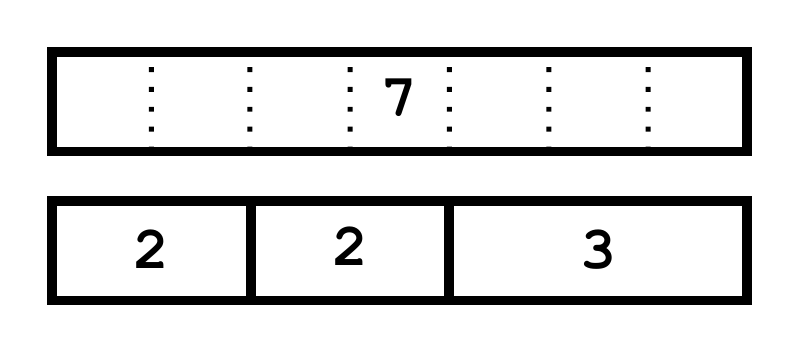}
\end{center}
\vspace{-0.7cm}
\caption{Segmentation as integer composition. An integer (top), arbitrary parts, and their composition (bottom) stand for a sequence, segments, and a possible segmentation.}
\vspace{-0.2cm}
\label{integer-composition}
\end{figure}

To obtain the growth rate for various restricted cases, we derive generating functions for each, whose coefficients count the number of compositions, and submit them to analysis based on the following lemma.

\begin{lemma}[Growth rate of the coefficients of a rational function]\label{combinatorics-lemma}
    Let $S(x)=\sum_{n \geq 0} s_n x^n = \frac{P(x)}{Q(x)}$ be a rational function with $Q(0) \neq 0$ and assume P(x) and $Q(x)$ do not have any roots in common. The general form of the coefficients is $[x^N]S(x) = A^N\Phi(N)$, where $A^N$ is the exponential growth factor and $\Phi(N)$ is the subexponential growth factor. Then the exponential growth rate $A$ of the sequence of coefficients $(s_n)$ is equal to $|\frac{1}{\alpha}|$, where $\alpha$ is the root of $Q(x)$ of smallest modulus \cite<for proof, see>[Theorem 7.10]{bonaCombinatorics2016}.
\end{lemma}

\subsubsection{Lower-bounded parts.} We consider integer compositions involving parts $p_i \in[a, N]$, with $1 < a < N$.

\begin{theorem}[]
    The number of $[a, N]$-restricted integer compositions of $N$ grows exponentially with $N$.
\end{theorem}

\begin{proof}
For each part $p_i$ the choice is among the positive integers ($a, a+1, a+2, ...$), so in terms of the generating function we have $(x^a + x^{a+1} + ...)$, where $a$ is an integer constant. By factoring and using the closed form of the geometric series we write $x^a(1 + x + x^2 ...) = x^a \frac{1}{1-x}$. For a $k$-part composition we thus have $(\frac{x^a}{1-x})^k$, so for compositions of any number of parts we sum\footnote{Note that summing from $k=0$ or from $k=1$ to $n$ does not affect the denominator of the resulting generating function, therefore the results will be robust to these variations.} over $k$ and write $\sum_{k=1}^\infty (\frac{x^a}{1-x})^k$, which similarly as above can be simplified $\frac{x^a}{1-x} \left[ 1 + (\frac{x^a}{1-x})^1 + (\frac{x^a}{1-x})^2 + ... \right]$ until we arrive at the closed form that completes the construction:
$$G_{LB}(x) := \frac{x^a}{1-x-x^a}$$

\noindent Having derived the generating function, we are interested in the growth rate of the coefficients $[x^N]G_{LB}(x)$. Let $Q(x) = 1 - x - x^a$, and note that $\lim_{x \to 0}Q(x)=1$ and $\lim_{x \to 1}Q(x)=-1$. Since $Q(0-\epsilon_1) > 0$ and $Q(1-\epsilon_2) < 0$ for some small positive $\epsilon_1, \epsilon_2$, by the intermediate value theorem, it follows that there always exists a root that satisfies $0 < \alpha < 1$. By Lemma \ref{combinatorics-lemma} the exponential growth factor is $A^N$, for some $A>1$.
\end{proof}

\subsubsection{Upper-bounded parts.} We consider integer compositions involving parts $p_i \in[1, b]$, with $1 < b < N$.

\begin{theorem}[]
    The number of $[1, b]$-restricted integer compositions of $N$ grows exponentially with $N$.
\end{theorem}

\begin{proof}
For each part $p_i$ the choice is among the positive integers ($1, 2, 3, ..., b$), so in terms of the generating function we have $(x^1 + x^{2} + ... + x^b)$, where $b$ is an integer constant. Factoring out $x$ and using the identity $1+y+y^2+...+y^n=\frac{1-y^{n+1}}{1-y}$, we write $x(1 + x^{1} + ... + x^{b-1}) = \frac{x-x^{b}}{1-x}$. 
For $k$-part compositions we raise to the $k$-th power and for compositions of any number of parts we sum over all $k$, which yields
$\sum_{k=0}^\infty (\frac{x-x^{b}}{1-x})^k = 1 + (\frac{x-x^{b}}{1-x})^1 + (\frac{x-x^{b}}{1-x})^2 + ... $, and finally by using the geometric series and simplifying we complete the construction:
$$G_{UB}(x) :=\frac{1-x}{1-2x+x^{b+1}}$$
Having derived the generating function, we are interested in the growth rate of the coefficients $[x^N]G_{UB}(x)$. Following the intermediate value theorem, let $Q(x)=1-2x+x^{b+1}$ and note that $\lim_{x \to 0}Q(x) > 0$, so $Q(0-\epsilon) > 0$ for some small positive $\epsilon$, and  $Q(\frac{3}{4}) < 0$, for any $b$ as above, hence there always exists a root $\alpha$ located in the interval $(0, \frac{3}{4}]$. By Lemma \ref{combinatorics-lemma} the exponential growth factor is $A^N$, for some $A>1$
\end{proof}

\subsubsection{Doubly-bounded parts.}  We consider compositions involving parts parts $p_i \in[a, b]$, with $1 < a < b < N$.

\begin{theorem}[]
    The number of $[a, b]$-restricted integer compositions of $N$ grows exponentially with $N$.
\end{theorem}

\begin{proof}
For each part $p_i$ the choice is among the positive integers ($a, a+1, ..., b$), so in terms of the generating function we have $(x^a + x^{a+1} + ... + x^b)$. Factoring out $x^a$ and using $1+y+y^2+...+y^n=\frac{1-y^{n+1}}{1-y}$, we write $\frac{x^a-x^{b+1}}{1-x}$. For compositions of any number $k$ of parts, we have $\sum_{k=1}^\infty (\frac{x^a-x^{b+1}}{1-x})^k = 1 + (\frac{x^a-x^{b+1}}{1-x})^1 + (\frac{x^a-x^{b+1}}{1-x})^2 + ... $
Using the geometric series and simplifying, we arrive at the final form:
$$G_{DB}(x) := \frac{x^a-x^{b+1}}{1-x-x^a+x^{b+1}}$$
\noindent Having derived the generating function, we are interested in the growth rate of the coefficients $[x^N]G_{DB}(x)$. Since $1 < a < b$, with $c = b-a \geq 2$, we write the polynomial in the denominator $Q(x) = 1 - x - x^a + x^{a+c}$ which we then rewrite as $(1-x)+x^a(x^c-1) = (1-x)(1-x^a[x^{c-1} + x^{c-2} + ... + 1])$. Note that $Q(x)$ has a root in $(0, 1)$ if and only if $q(x) = 1-x^a(x^{c-1} + x^{c-2} + ... + 1)$ has such a root. We have that $\lim_{x \to 0 }q(x) = 1$ and furthermore $\lim_{x \to 1 }q(x) = 1 - c$, and recall $c \geq 2$. By the intermediate value theorem, since $q(0-\epsilon_1) > 0$ and $q(1-\epsilon_2) < 0$ for some small positive $\epsilon_1, \epsilon_2$, there always exists a root $\alpha$ of $Q(x)$ in the open interval $(0, 1)$ for any integers $a,b$ constrained as above. By Lemma \ref{combinatorics-lemma} the exponential growth factor is $A^N$, for some $A>1$.
\end{proof}

\subsection{Hardness of segmentation}

We showed that intuitive constraints do not render brute-force segmentation tractable. One may be tempted to conclude that this demonstrates the conjectured hardness of the segmentation problem. However, in this section we present a theorem that contradicts this conclusion. The proof builds on the technique of (polynomial-time) reduction \cite{aroraComputationalComplexityModern2009,gareyComputersIntractabilityGuide1979,vanrooijCognitionIntractabilityGuide2019}.  

%\begin{definition}(Tractability)
%    A problem is said to be \textit{tractable} if there exists at least one polynomial-time algorithm that can solve it. Otherwise it is called \textit{intractable}.
%\end{definition}

\begin{definition}(Polynomial-time reducibility)
    Let $A$ and $B$ be computational problems. We say $A$ is 
    \textit{polynomial-time reducible} to $B$ if it is possible to tractably transform instances of $A$ into instances of $B$ such that solutions for $B$ can be easily transformed into solutions for $A$. Note that this implies that if a tractable algorithm for $B$ exists, it could be used to solve $A$ tractably (namely, via the tractable transformation, called the \textit{reduction}).
\end{definition}

\noindent We present such a reduction from the problem \textsc{segmentation} to a problem in graph theory. On the way, we will have introduced an alternative way of thinking about segmentation at the computational and algorithmic levels.

\begin{theorem}[]
    \textsc{segmentation} is tractable (polynomial-time computable) in the absence of constraints.
\end{theorem}

\begin{proof}
We will show that, given an arbitrary instance of the segmentation problem, we can tractably construct an instance (with the correct associated output) of a target problem which is itself tractably computable. To begin, we introduce a class of graphs which we use as a stepping stone.

\begin{definition}[Interval Graph]\label{interval-graph-def}
     An interval graph is an undirected graph $G = (V, E)$ built from a collection of intervals $\{p_i\} = \{\{x|a_i<x<b_i\}, ...\}$, here $x,a_i,b_i \in \mathbb{Z}$, by creating one vertex $v_i \in V$ for each interval $p_i$ and an edge $\{v_i, v_j\}$ whenever the corresponding intervals have a non-empty intersection: $E = \{\{v_i, v_j\} \in V \times V \ |\  p_i \cap p_j \neq \emptyset\}$.
 \end{definition}

\noindent Consider an instance of \textsc{segmentation}. Given an input sequence, it is possible to construct an \textit{interval graph} that satisfies Def. \ref{interval-graph-def}. Algorithm \ref{alg:cap} demonstrates the procedure.

\begin{algorithm}[!ht]
    \caption{\small
        Construct segment graph from sequence. \\
        \noindent\rule{0.482\textwidth}{1pt}\\
        \textit{Input:} \\ 
        $S = (s_1, s_2, s_3, ..., s_N), \ s_i \in \mathbb{Z}$ \Comment{\texttt{sequence}} \\
        $F: \mathscr{P} \mapsto \mathbb{Z^+}$   \Comment{\texttt{scoring function}} \\
        \textit{Output:}\\ 
        $V = \{(v_1, w_1), ..., (v_q, w_q)\}$    \Comment{\texttt{weighted vertex set}} \\
        $E = \{(v_i, v_j), ...\}$    \Comment{\texttt{edge set}}
    }\label{alg:cap}
    
    \begin{algorithmic}[1]\small 
        \item[]
        \Function{BuildSegmentGraph}{$S$, $F$}
        \State $P \gets \text{[ ]}$ \Comment{\texttt{legal segments}}
        \State $V \gets \{\}$   \Comment{\texttt{weighted vertex set}}
        \State $E \gets \{\}$   \Comment{\texttt{edge set}}
        
        % \item[]
        \item[\Comment{\texttt{construct weighted vertex set}}] % unnumbered line
        
        \For{$i \gets 1$ to $|S|$}
            \For{$j \gets 0$ to $|S|-i$}
    			\State $segment \gets S[i:i+j]$
    			\State $weight \gets F(segment) \times (-1)$
    			\State $P\textbf{.append}([i, ..., i+j])$
    			\State $V\textbf{.append}(([i, ..., i+j], weight))$
			\EndFor
        \EndFor
        
        % \item[]
        \item[\Comment{\texttt{construct edge set}}] % unnumbered line
    	
    	\For{$i \gets 1$ to $|P|-1$}
		    \For{$j \gets i+1$ to $|P|$}
    			\If{$P[i] \cap P[j] \neq \emptyset$}
    				\State $E\textbf{.append}((P[i], P[j]))$
				\EndIf
            \EndFor
        \EndFor

        % \item[]

        \State \Return $(V, E)$
        \EndFunction
    \end{algorithmic}
\end{algorithm}
\vspace{-0.3cm}

\begin{remark}
Algorithm \ref{alg:cap} involves systematically generating all legal segments, computing and negating their weights, checking their pairwise overlap, and using this to construct a graph. We call this object a \textit{segment graph}.
\end{remark}

Consider the \textit{time complexity} of Algorithm \ref{alg:cap}. The elementary instructions are the weight computation (line 8), appending (lines 9-10, 16), and set intersection (line 15); all of which are polynomial-time computable ($F$ is assumed to be). We focus now on the number of implied iterations.
The loops defined in lines 5-6 yield $N + (N-1) + \ ... \ + 1$ iterations (the number of possible segments), given by a polynomial:
	\begin{align*}
	|P_N| &= \sum_{k=1}^N k = \frac{N(N+1)}{2}  &(N^{\text{th}} \text{ triangular number})
	\end{align*}
The loops defined in lines 13-14 yield a number of iterations equal to the number of segment pairs $(p_i, p_j) \in P_N^*$, given by the binomial coefficient $\binom{n}{k}$ with $n=|P_N|$ and $k=2$, which grows as a quadratic in $|P_N|$ (i.e. $4^\text{th}$-degree polynomial in $N$).
	$$|P_N^*| = \binom{|P_N|}{2} \sim O(N^4)$$
	
\noindent This algorithmic analysis demonstrates that \textsc{BuildSegmentGraph} (Algorithm \ref{alg:cap}) is polynomial-time computable.

Consider next the \textit{correctness} of Algorithm \ref{alg:cap}. We will show that a segment graph encodes the properties of candidate solutions to an instance of \textsc{segmentation}. For this, we need the following definitions.

\begin{definition}[Independent sets and maximality]\label{independence}
    Let $G = (V, E)$ denote a graph. We call a vertex set $V^* \subseteq V$ an \textit{independent set} if there exist no two vertices $u, v \in V^*$ such that $(u, v) \in E$. Such a set is said to be \textit{maximal} if there exists no vertex $v \in V$ that can be added to $V^*$ without breaking the independence.
    %(i.e., $V^* \cup \{v\}$ is still an independent set).
\end{definition}

\begin{definition}[Dominating sets and minimality]\label{dominance}
    Let $G = (V, E)$ denote a graph. We call a vertex set $V^* \subseteq V$ a \textit{dominating set} if for all $v \in V$, either $v \in V^*$ or  there is an edge $(v, u) \in E$ for some $u \in V^*$. Such a set is said to be \textit{minimal} if there exists no vertex $v \in V^*$ that can be removed without breaking the dominance. 
    %(i.e., $V^* \setminus \{v\}$ is still a dominating set).
\end{definition}

\noindent By construction, a legal segmentation (i.e., a collection of disjoint segments whose concatenation yields the original sequence) is guaranteed to be represented within the segment graph as a subset of vertices with two properties: 
\vspace{-0.0cm}
\begin{itemize}
    \itemsep0.2em
    \item \textit{maximal independence}: vertices are pairwise non-adjacent because segments in a segmentation should be disjoint; since the segments should span the sequence, adding any vertex breaks independence.
    \item \textit{minimal dominance}: vertices in the graph are either in the subset or adjacent to one of its elements because once a segment subset spans the sequence, any other segment is guaranteed to overlap; since the segments should be disjoint, removing any vertex breaks dominance.
\end{itemize}

\begin{remark}
How segment graphs make the structure of the original sequence problem transparent is illustrated in Fig. \ref{segment-graph}.
    
\end{remark}

\begin{figure}[h]
\vspace{0.0cm}
\begin{center}
    \includegraphics[scale=0.42]{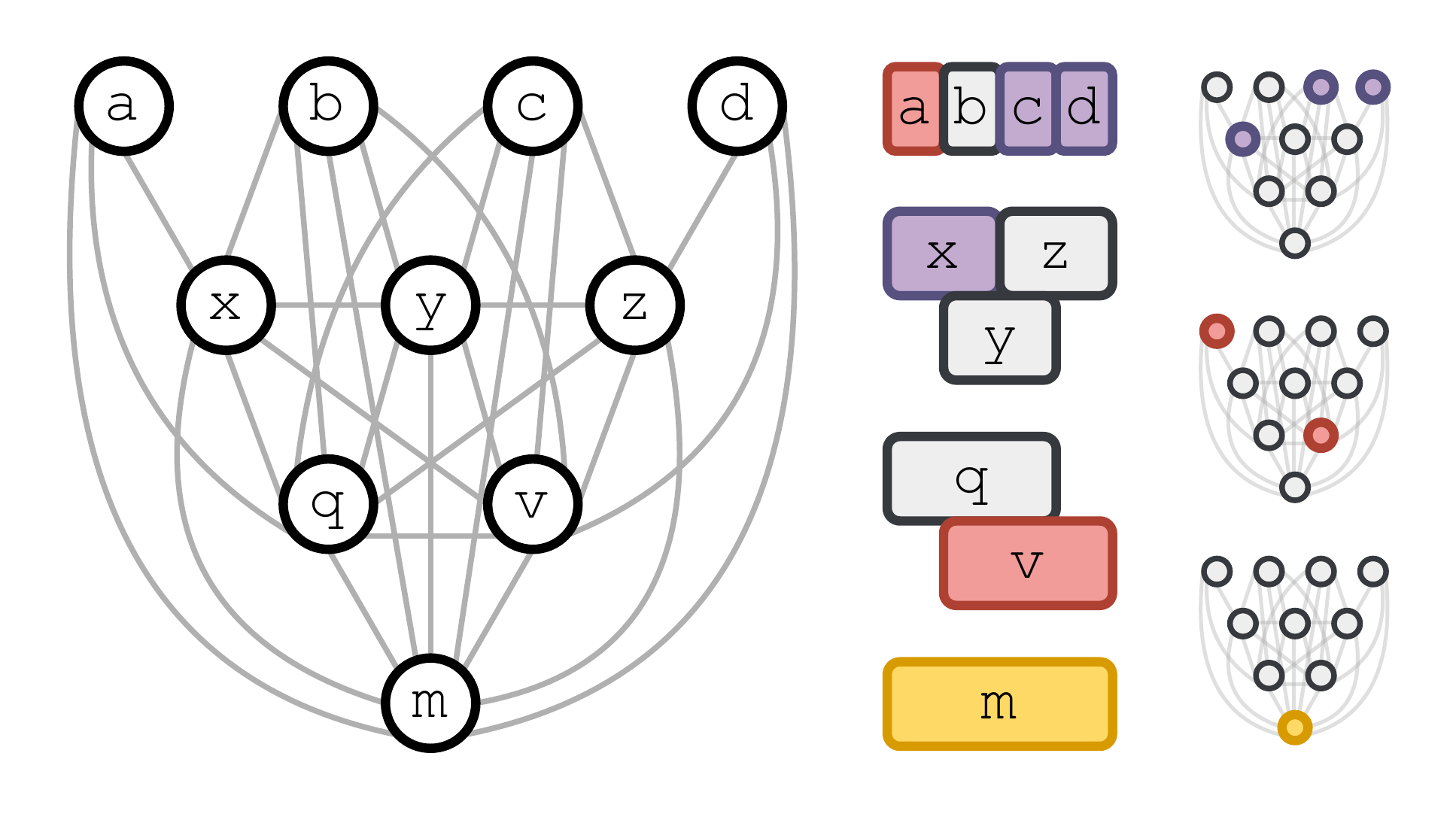}
\end{center}
\vspace{-0.5cm}
\caption{Segment graph encoding (left) of a sequence of length $4$. Nodes correspond to possible segments and edges represent pairwise overlap in the sequence. Possible segments are grouped according to length and overlap (middle). Candidate solutions are vertex subsets in the segment graph (right). Three example segmentations are color coded.}
\vspace{-0.3cm}
\label{segment-graph}
\end{figure}

\noindent A general feature of dominance and independence on arbitrary graphs is useful:

\begin{lemma}\label{ind-dom-lemma}
    An independent vertex set in a graph is a dominating set if and only if it is a \textit{maximal} independent set. Any such set is necessarily also a \textit{minimal} dominating set. \cite<cf.>{berge1962theory,goddardIndependentDominationGraphs2013}.
\end{lemma}

\noindent It follows from the above and Lemma \ref{ind-dom-lemma} that if a vertex subset in a segment graph is independent and dominant, then it is a candidate solution (i.e., valid segmentation). A feasible solution has, additionally, minimum weight among candidates: it is a \textit{minimum-weight independent dominating set}. With this, we introduce the formal graph problem we reduce to.
\newpage
\begin{definition} \ \\
    \textsc{minimum-weight independent dominating set} \\
    \textit{Input}: A vertex-weighted graph $G=(V, E)$. For each $v \in V$ we have a weight $W(v) \in \mathbb{Z}$. \\
    \textit{Output}: An independent dominating set $V^* \subseteq V$  such that $Q(V^*) = \sum_{v \in V^*}W(v)$ is minimum.
    
    % \noindent (Call a vertex set an \textit{independent dominating set} if it has both the independence (Def. \ref{independence}) and dominance (Def. \ref{dominance}) properties as defined above).
\end{definition}

\noindent So far, we have established that, given an instance $I_{seq}$ of \textsc{segmentation}, we can  construct, in polynomial time by Algorithm \ref{alg:cap}, call it $A(.)$, a corresponding instance $I_{graph} = A(I_{seq})$ of \textsc{Minimum-weight independent dominating set}. This demonstrates the validity of the reduction and we now finally consider the tractability of the problems.

Though the problem of finding minimum-weight independent dominating sets is \textit{NP-hard} in general and remains so in several special cases \cite{gareyComputersIntractabilityGuide1979,liuIndependentDominatingSet2015}, the following input restriction is relevant. 

\begin{lemma}\label{mwids-lemma}
    \textsc{minimum-weight independent dominating set} is polynomial-time computable provided the input graph is an \textit{interval graph}. \cite<for proof, see>[Theorem 2.4]{changEfficientAlgorithmsDomination1998}.
\end{lemma}

\noindent Recall that the restriction required by Lemma \ref{mwids-lemma} is guaranteed by our reduction. Hence, we conclude \textsc{segmentation} is tractably computable, which completes the proof.
\end{proof}

\section{Discussion}

Computational feasibility is a widespread concern that motivates choices in the framing and modeling of biological and artificial intelligence. While implicit or informal assumptions abound, the reality may turn out to be counterintuitive as they are examined formally. Here, we undertook a formal examination of the existing computational assumptions about \textit{Segmentation}. Using complexity-theoretic tools, we mathematically proved two sets of results that run counter to commonly-held assumptions: 1) the search space is either not large to begin with or it is large but placing intuitive constraints does not alleviate the issue; and 2) a computational model of segmentation that formalizes its conceptualization across domains is tractably computable in the absence of widely-adopted constraints to address the assumed hardness.

Beyond our proofs, we set the groundwork for further refinements of segmentation theory and its computational analyses: a) we contributed a formalization of the computation that satisfies a domain-agnostic specification; b) we illustrated the relationship between segmentation and integer compositions, which makes the search space amenable to asymptotic analyses; and c) we built a bridge from the problem as originally defined on sequences to the mathematics of graphs, which opens up alternative formalisms to model it and to think about it algorithmically. A desirable consequence of translating problems between formal domains is that structure which was originally hidden from view may become visible.

Our results challenge existing intuitions about hardness of the segmentation problem and its sources of complexity, and by extension question the motivation of proposed solutions and their associated empirical research foci. For instance, concerns about search space size and what mitigates it may be misplaced. The space of possible segments is not exponential to begin with; the space of segmentations is. However, the bounds on segment size are, neither individually nor combined, a source of exponentiality.  Left unexamined, this may still appear to support the conjectured hardness of the problem. But our tractability proof challenges this intuitive conclusion. It demonstrates that no assumptions about bottom-up segmentation cues or top-down biases on segment properties are necessary to make the formal problem tractable.  These proofs run counter to the computational efficiency concerns that partially motivate segmentation theories. For instance, proposals that argue from minimal units of representation \cite<cf.>{poppelHierarchicalModelTemporal1997}, temporal integration limits of neuronal populations \cite<cf.>{overathCorticalAnalysisSpeechspecific2015}, intrinsic oscillatory timescales \cite<cf.>{ghitzaRoleThetaDrivenSyllabic2012,wolffIntrinsicNeuralTimescales2022}, bottom-up segmentation cues \cite<e.g.,>{giraudCorticalOscillationsSpeech2012}, and top-down biases on candidate search \cite<e.g.,>{fristonActiveListening2021}, which to some extent build on the supposition of problem hardness, search space size, and various sources of complexity. This suggests that intractability concerns, if any, might be better placed, for instance, on the processes guiding segmentation rather than the boundary placement itself.

Together, the results proven here caution against intuitive notions about the properties of computational problems driving empirical programs, and demonstrate the need and benefits of critically assessing their soundness. Whenever intuitions are challenged, this enables researchers to either slightly or entirely redirect efforts as ideas shift regarding what evidence is relevant to collect. For instance, if researchers believe that a certain problem is computationally hard and that some set of neural and environmental regularities might speak to constraints that make it tractable, then they would be inclined to look for those regularities that satisfy such a requirement. If, however, the original belief is removed, the target regularities or the kinds of experiments that are adequate to test their putative role might be different.

We close with a similar word of caution about interpreting our results. These are to some degree tied to the particular formalization we put forth. While modeling choices were motivated and they bear some generality, alternative theoretical commitments are conceivable. For instance, an extended model could allow for multiple unsegregated high-dimensional input streams; it is an open question whether it would have different complexity properties. We view our analyses not as the last word on the computational complexity of segmentation but rather as initial words in a conversation with a sound formal basis.

\section{Acknowledgments}
We thank the Computational Cognitive Science group at the Donders Institute for Brain, Cognition, and Behaviour for discussions, Nils Donselaar for invaluable feedback on a previous version that helped improve the manuscript, and Ronald de Haan for comments on future directions of this work. We thank four anonymous reviewers and one anonymous meta-reviewer for thoughtful comments, and Reviewer 1 in particular for a comprehensive, constructive and educational review. FA thanks David Poeppel for support and discussions on auditory segmentation. TW was supported by NSERC Discovery Grant 228104-2015.

\nocite{baldwinSegmentingDynamicHuman2008}
\nocite{berge1962theory}

\bibliographystyle{apacite}

\setlength{\bibleftmargin}{.125in}
\setlength{\bibindent}{-\bibleftmargin}

\bibliography{references}

\end{document}